%% file: nips_2017_2_arxiv.tex
\documentclass[10pt,journal,onecolumn]{IEEEtran}

\usepackage[utf8]{inputenc} 
\usepackage[T1]{fontenc}    
\usepackage{hyperref}       
\usepackage{url}            
\usepackage{booktabs}       
\usepackage{amsfonts}       
\usepackage{nicefrac}       
\usepackage{microtype}      
\usepackage{amsfonts,amsmath,amsthm,amssymb}       
\usepackage[usenames,dvipsnames]{xcolor}
\usepackage{graphicx, forloop}
\usepackage[export]{adjustbox}

\usepackage{floatrow}
\input{settings_custom}
\title{Analysis of universal adversarial perturbations}

\author{
\IEEEauthorblockN{Seyed-Mohsen Moosavi-Dezfooli\IEEEauthorrefmark{1}\IEEEauthorrefmark{2}\thanks{\IEEEauthorrefmark{1}The first two authors contributed equally to this work.}\;\,\thanks{\IEEEauthorrefmark{2}LTS4, \'Ecole Polytechnique F\'ed\'erale de Lausanne, Switzerland}}
\IEEEauthorblockA{\texttt{seyed.moosavi@epfl.ch}}
\\
\IEEEauthorblockN{Alhussein Fawzi\IEEEauthorrefmark{1}\IEEEauthorrefmark{3}\;\,\thanks{\IEEEauthorrefmark{3}UCLA Vision Lab, University of California, Los Angeles, CA 90095}}
  \IEEEauthorblockA{\texttt{fawzi@cs.ucla.edu}}
  \\
  \IEEEauthorblockN{Omar Fawzi\IEEEauthorrefmark{4}\thanks{\IEEEauthorrefmark{4}ENS de Lyon, LIP, UMR 5668 ENS Lyon - CNRS - UCBL - INRIA, Universit\'e de Lyon, France}}
  \IEEEauthorblockA{\texttt{omar.fawzi@ens-lyon.fr}}
  \\
  \IEEEauthorblockN{Pascal Frossard\IEEEauthorrefmark{2}}
  \IEEEauthorblockA{\texttt{pascal.frossard@epfl.ch}}
  \\
  \IEEEauthorblockN{Stefano Soatto\IEEEauthorrefmark{3}}
  \IEEEauthorblockA{\texttt{soatto@cs.ucla.edu}}
}

\begin{document}

\maketitle

\begin{abstract}
Deep networks have recently been shown to be vulnerable to universal perturbations: there exist very small \textit{image-agnostic} perturbations that cause most natural images to be misclassified by such classifiers. In this paper, we propose a  quantitative analysis of the robustness of classifiers to universal perturbations, and draw a formal link between the robustness to universal perturbations, and the geometry of the decision boundary.  Specifically, we establish theoretical bounds on the robustness of classifiers under two decision boundary models (\textit{flat} and \textit{curved} models). We show in particular that the robustness of deep networks to universal perturbations is driven by a key property of their curvature: there exists shared directions along which the decision boundary of deep networks is systematically positively curved. Under such conditions, we prove the existence of small universal perturbations. Our analysis further provides a novel geometric method for computing universal perturbations, in addition to explaining their properties.
\end{abstract}

\section{Introduction}
Despite the recent success of deep neural networks in solving complex visual tasks \cite{krizhevsky2012imagenet}, these classifiers have recently been shown to be highly vulnerable to perturbations in the input space. In \cite{moosavi2017universal}, state-of-the-art deep networks are empirically shown to be vulnerable to \textit{universal} perturbations: there exist very small image-agnostic perturbations that cause most natural images to be misclassified. These perturbations fundamentally differ from the random noise regime, and exploit essential properties of the deep network to misclassify most natural images with perturbations of very small magnitude. Why are state-of-the-art classifiers highly vulnerable to these specific directions in the input space? What do these directions represent? To answer these questions, we follow a theoretical approach and find the causes of this vulnerability in the geometry of the decision boundaries induced by deep neural networks. For deep networks, we show that the key to answering these questions lies in the existence of \textit{shared directions} (across different datapoints) along which the decision boundary is highly positively curved. This establishes fundamental connections between geometry and robustness to universal perturbations, and thereby reveal unknown properties of the decision boundaries induced by deep networks.





While \cite{moosavi2017universal} studies the vulnerability to universal perturbations of deep neural networks, our aim here is to derive a more generic analysis in terms of the geometric properties of the boundary. To this end,
we introduce two decision boundary models: 1) the \textit{locally flat} model assumes that the first order linear approximation of the decision boundary holds locally in the vicinity of the natural images, and 2) the \textit{locally curved} model provides a second order description of the decision boundary, and takes into account the \textit{curvature} information. We summarize our contributions as follows:
\begin{itemize}
\item Under the \textit{locally flat} decision boundary model, we show that classifiers are vulnerable to universal directions so long as the normals to the decision boundaries in the vicinity of natural images are \textit{correlated} (i.e., approximately span a low dimensional space). This result formalizes and proves the empirical observations made in \cite{moosavi2017universal}.
\item Under the \textit{locally curved} decision boundary model, the robustness to universal perturbations is instead driven by the \textit{curvature} of the decision boundary; we show that the existence of \textit{shared} directions along which the decision boundary is positively curved implies the existence of very small universal perturbations. 
\item We show that state-of-the-art deep nets satisfy the assumption of our theorem derived for the locally curved model: there exist shared directions along which the decision boundary of deep neural networks are positively curved. Our theoretical result consequently captures the large vulnerability of state-of-the-art deep networks to universal perturbations. 
\item We finally show that the developed theoretical framework provides a novel (geometric) method for computing universal perturbations, and further explains some of the properties observed in \cite{moosavi2017universal} (e.g., diversity, transferability).
\end{itemize}
The robustness of classifiers to noise has recently gained a lot of attention.  Several works have focused on the analysis of the robustness properties of SVM classifiers (e.g., \cite{xu2009robustness}) and new approaches for constructing robust classifiers (based on robust optimization) have been proposed \cite{sra2012optimization, lanckriet2003robust}. More recently, several works have assessed the robustness of deep neural networks to different regimes such as adversarial perturbations \cite{szegedy2013intriguing}, random noise \cite{nips2016_ours}, and realistic occlusions \cite{sharif2016accessorize}. The robustness of classifiers to adversarial perturbations has been specifically studied, leading to theoretical as well empirical explanations of the phenomenon \cite{fawzi2015analysis, goodfellow2014, tabacof2015exploring, tanay2016boundary}. Although related, adversarial perturbations  are fundamentally different from universal perturbations, as the latter is a single perturbation that is added to all images, while the former is adapted to each image (and adversarial perturbations are not universal as shown in \cite{moosavi2017universal}). While several works attribute the high vulnerability to perturbations of deep nets to their low flexibility \cite{fawzi2015analysis}, high linearity \cite{goodfellow2014}, or flat decision boundaries \cite{nips2016_ours}, we show here that it is, on the contrary, the large curvature of the decision boundary that causes the vulnerability to universal perturbations. Our bounds indeed show an increasing vulnerability with respect to the curvature of the decision boundary, and represent the first formal result showing the tight links between robustness and large curvature. Finally, it should be noted that few works have recently studied properties of deep networks from a geometric perspective (such as their expressivity \cite{poole2016exponential, montufar2014number}); our focus is however different as we analyze the robustness with the geometry of the decision boundary.

\vspace{-2mm}

\section{Definitions and notations}
\label{sec:definitions}

Consider an $L$-class classifier $f:\mathbb{R}^d\rightarrow\mathbb{R}^L$. Given a datapoint $\x_0 \in \mathbb{R}^d$, we define the estimated label $\lab(\x_0)=\argmax_k f_k(\x_0)$,
where $f_k(\x)$ is the $k$th component of $f(\x)$ that corresponds to the $k^{\text{th}}$ class.
We define by $\mu$ a distribution over natural images in $\bb{R}^d$. The main focus of this paper is to analyze the robustness of classifiers to \textit{universal} (image-agnostic) noise. Specifically, we define $\v$ to be a \textit{universal} noise vector if $\hat{k} (\x+\v) \neq \hat{k} (\x) \text{ for ``most'' } \x \sim \mu.$ Formally, a perturbation $\v$ is $(\xi, \delta)$-universal, if the following two constraints are satisfied:
\begin{align*}
\| \v \|_2 & \leq \xi, \\
\Pbb\left( \hat{k} (\x+\v) \neq \hat{k} (\x) \right) & \geq 1 - \delta.
\end{align*}
This perturbation image $\v$ is coined ``universal'', as it represents a fixed image-agnostic perturbation that causes label change for a large fraction of images sampled from the data distribution $\mu$.
In \cite{moosavi2017universal}, state-of-the-art classifiers have been shown to be surprisingly vulnerable to this simple perturbation regime. 

It should be noted that universal perturbations are different from \textit{adversarial perturbations} \cite{szegedy2013intriguing}, which are datapoint specific perturbations that are sought to fool a \textit{specific} image. An adversarial perturbation is specifically defined as the solution of the following optimization problem
\begin{align}
\label{eq:adv_pert}
\r(\x) = \arg\min_{\r \in \mathbb{R}^d} \| \r \|_2 \text{ subject to } \hat{k} (\x+\r) \neq \hat{k} (\x),
\end{align}
which corresponds to the smallest additive perturbation that is necessary to change the label of the classifier $\hat{k}$ for $\x$. From a geometric perspective, $\r(\x)$ quantifies the distance from $\x$ to the decision boundary (see Fig. \ref{fig:adversarial_geometry}). In addition, due to the optimality conditions of Eq. (\ref{eq:adv_pert}), $\r(\x)$ is orthogonal to the decision boundary at $\x+\r(\x)$, as illustrated in Fig. \ref{fig:adversarial_geometry}.


In the remainder of the paper, we analyze the robustness of classifiers to universal noise, with respect to the geometry of the \textit{decision boundary} of the classifier $f$. Formally, the pairwise decision boundary, when restricting the classifier to class $i$ and $j$ is defined by $\mathscr{B} = \{ \z \in \bb{R}^d: f_{i} (\z) - f_{j} (\z) = 0\}$ (we omit the dependence of $\mathscr{B}$ on $i,j$ for simplicity). The decision boundary of the classifier hence corresponds to points in the input space that are equally likely to be classified as $i$ or $j$. 

In the following sections, we introduce two models on the decision boundary, and quantify in each case the robustness of such classifiers to universal perturbation. We then analyze the robustness of deep networks in light of the two models. 

\begin{figure}
  \centering
  \begin{subfigure}[b]{0.49\textwidth}
    \center
    \includegraphics[width=0.6\textwidth]{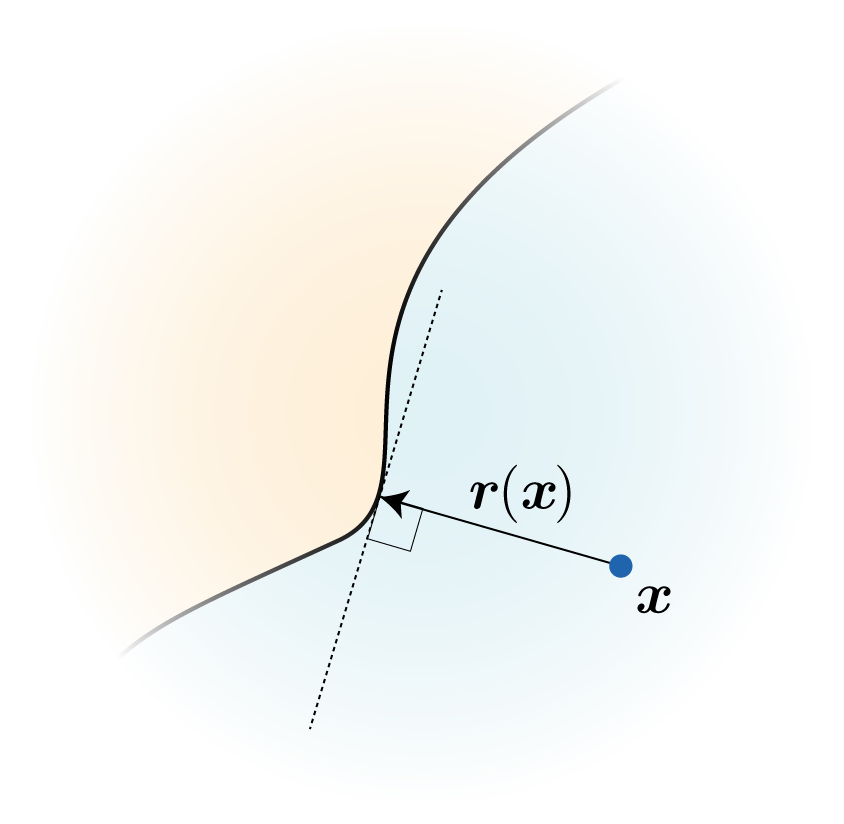}
    \caption{Local geometry of the decision boundary.}
    \label{fig:adversarial_geometry}
  \end{subfigure}
  ~
  \begin{subfigure}[b]{0.49\textwidth}
    \center
    \includegraphics[width=0.7\textwidth]{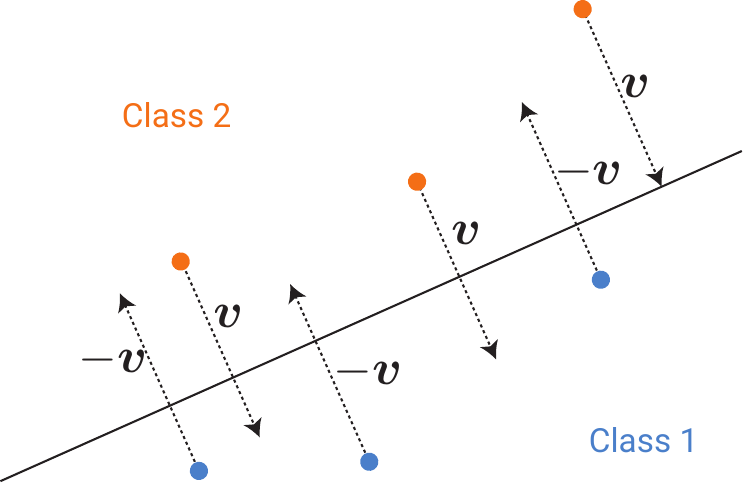}
    \caption{Universal direction $\v$ of a linear binary classifier.}
    \label{fig:binary_classifier}
  \end{subfigure}
  \caption{}
\end{figure}

\section{Robustness of classifiers with flat decision boundaries}
Several recent works have attempted to explain the vulnerability of classifiers to different types of perturbations (e.g., random noise, and adversarial perturbations) by assuming a flat decision boundary model. We start here our analysis with a similar hypothesis, and analyze the robustness of classifiers to universal perturbations under this decision boundary model. 
\label{sec:locally_flat_model}
We specifically study the existence of a universal direction $\v$, such that
\begin{align}
\label{eq:fooling_condition}
\hat{k} (\x+\v) \neq \hat{k} (\x) \text{ or } \hat{k} (\x-\v) \neq \hat{k} (\x),
\end{align}
where $\v$ is a vector of sufficiently small norm. It should be noted that a universal \textit{direction} (as opposed to a universal vector) is sought in Eq. (\ref{eq:fooling_condition}), as this definition is more adapted to the analysis of \textit{linear} classifiers. For example, while a \textit{binary} linear classifier has a universal direction that fools all the data points, only half of the data points can be fooled with a universal vector (provided the classes are balanced) (see Fig.~\ref{fig:binary_classifier}). We therefore consider this slightly modified definition in the remainder of this section. 

We start our analysis by introducing our local decision boundary model. For $\x \in \Rbb^d$, note that $\x+\r(\x)$ belongs to the decision boundary and $\r(\x)$ is normal to the decision boundary at $\x + \r(\x)$ (see Fig. \ref{fig:adversarial_geometry}). A linear approximation of the decision boundary of the classifier at $\x+\r(\x)$ is therefore given by $\x + \{ \v: \r(\x)^T \v = \| \r(\x) \|_2^2 \}$. 
Under this approximation, the vector $\r(\x)$ hence captures the local geometry of the decision boundary in the vicinity of datapoint $\x$. We assume a local decision boundary model in the vicinity of datapoints $\x \sim \mu$, where the local classification region of $\x$ occurs in the halfspace $\r(\x)^T \v \leq \| \r(\x) \|_2^2$. Equivalently, we assume that outside of this half-space, the classifier outputs a different label than $\hat{k} (\x)$. However,  since we are analyzing the robustness to universal \textit{directions} (and not vectors), we consider the following condition, given by
\begin{align}
\mathscr{L}_{s}(\x, \rho): \forall \v \in B(\rho), | \r(\x)^T \v | \geq \| \r(\x) \|_2^2 \implies \hat{k} (\x+\v) \neq \hat{k} (\x) \text{ or } \hat{k} (\x-\v) \neq \hat{k} (\x).
\end{align}
where $B(\rho)$ is a ball of radius $\rho$ centered at $\mathbf{0}$. An illustration of this decision boundary model is provided in Fig. \ref{fig:linear_model}. It should be noted that linear classifiers satisfy this decision boundary model, as their decision boundaries are globally flat. This \textit{local} decision boundary model is however more general, as we do \textit{not} assume that the decision boundary is linear, but rather that the classification region in the vicinity of $\x$ is included in $\x + \{ \v: | \r(\x)^T \v | \leq \| \r(\x) \|_2^2\}$. Fig. \ref{fig:linear_model} provides an example of nonlinear decision boundary that satisfies this model.

In all the theoretical results of this paper, we assume that $\| \r(\x) \|_2 = 1$,  for all $\x \sim \mu$, for simplicity of the exposition. The results can be extended in a straightforward way to the case where $\| \r(\x) \|_2$ takes different values for points sampled from $\mu$. The following result shows that classifiers following the flat decision boundary model are \textit{not} robust to small universal perturbations, provided the normals to the decision boundary (in the vicinity of datapoints)  approximately belong to a low dimensional subspace of dimension $m\ll d$.

\begin{theorem}
\label{thm:theorem_S}
Let $\xi \geq 0, \delta \geq 0$. Let $\mathcal{S}$ be an $m$ dimensional subspace such that$\| P_{\mathcal{S}} \r(\x) \|_2 \geq 1 - \xi \text{ for almost all } \x \sim \mu,$, where $P_{\mathcal{S}}$ is the projection operator on the subspace. Assume moreover that $\mathscr{L}_s \left(\x, \rho\right)$ holds for almost all $\x \sim \mu$, with $\rho = \frac{\sqrt{e m}}{\delta(1-\xi)}$. Then, there exists a universal noise vector $\v$, such that $\| \v \|_2 \leq \rho$ and $\pr{x\sim\mu}{\fe} \geq 1 - \delta.$
\end{theorem}


\vspace{-2mm}

The proof can be found in the appendix, and relies on the construction of a universal perturbation through randomly sampling from $\mathcal{S}$.  The vulnerability of classifiers to universal perturbations can be attributed to the \textit{shared} geometric properties of the classifier's decision boundary in the vicinity of different data points. In the above theorem, this shared geometric property across different data points is expressed in terms of the normal vectors $\r(\x)$ in the neighborhood of the decision boundary. The main assumption of the above theorem is specifically that normal vectors $\r(\x)$ in the neighborhood of the decision boundary approximately live in a subspace $\mathcal{S}$ of low dimension $m<d$. Under this assumption, 
the above result shows the existence of universal perturbations of $\ell_2$ norm of order $\sqrt{m}$. When $m \ll d$, Theorem \ref{thm:theorem_S} hence shows that very small (compared to random noise, which scales as $\sqrt{d}$ \cite{nips2016_ours}) universal perturbations misclassifying most data points can be found. 

\textbf{Remark 1.} Theorem \ref{thm:theorem_S} can be readily applied to assess the robustness of multiclass linear classifiers to universal perturbations. In fact, when $f(\x)=W^T\x$, with $W = [\w_1, \dots, \w_L]$, the normal vectors are equal to $\w_i - \w_j$, for $1 \leq i, j \leq L, i \neq j$. These normal vectors exactly span a  subspace of dimension $L-1$. Hence, by applying the result with $\xi = 0$, and $m = L-1$, we obtain that linear classifiers are vulnerable to universal noise, with magnitude proportional to $\sqrt{L-1}$. In typical problems, we have $L \ll d$, which leads to very small universal directions. 

\textbf{Remark 2.} Theorem \ref{thm:theorem_S} formalizes the empirical observations made in \cite{moosavi2017universal} and therefore provides a partial expalantion to the vulnerability of deep networks, provided a locally flat decision boundary model is assumed (which is the assumption in e.g., \cite{nips2016_ours}). In fact, normal vectors in the vicinity of the decision boundary of deep nets have been observed to approximately span a subspace $\mathcal{S}$. However, unlike linear classifiers, the dimensionality of this subspace $m$ is typically larger than the the number of classes $L$, leading to large upper bounds on the norm of the universal noise, under the flat decision boundary model. This simplified model of the decision boundary hence fails to exhaustively explain the large vulnerability of state-of-the-art deep neural networks to universal perturbations. 

We show in the next section that the second order information of the decision boundary contains crucial information (\textit{curvature}) that captures the high vulnerability to universal perturbations.

\begin{figure}[ht]
  \center
  \begin{subfigure}{0.49\textwidth}
    \center
  \includegraphics[width=0.6\textwidth]{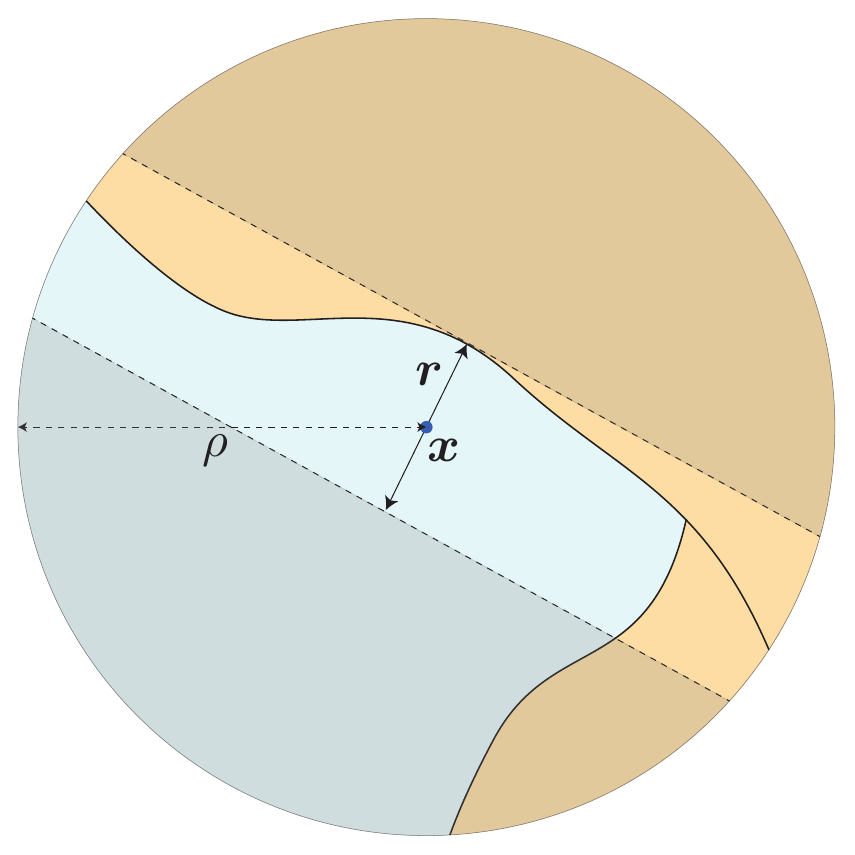}
  \caption{Flat decision boundary model $\mathscr{L}_s (\x, \rho)$.}
  \label{fig:linear_model}
  \end{subfigure}
  ~
  \begin{subfigure}{0.49\textwidth}
    \center
  \includegraphics[width=0.6\textwidth]{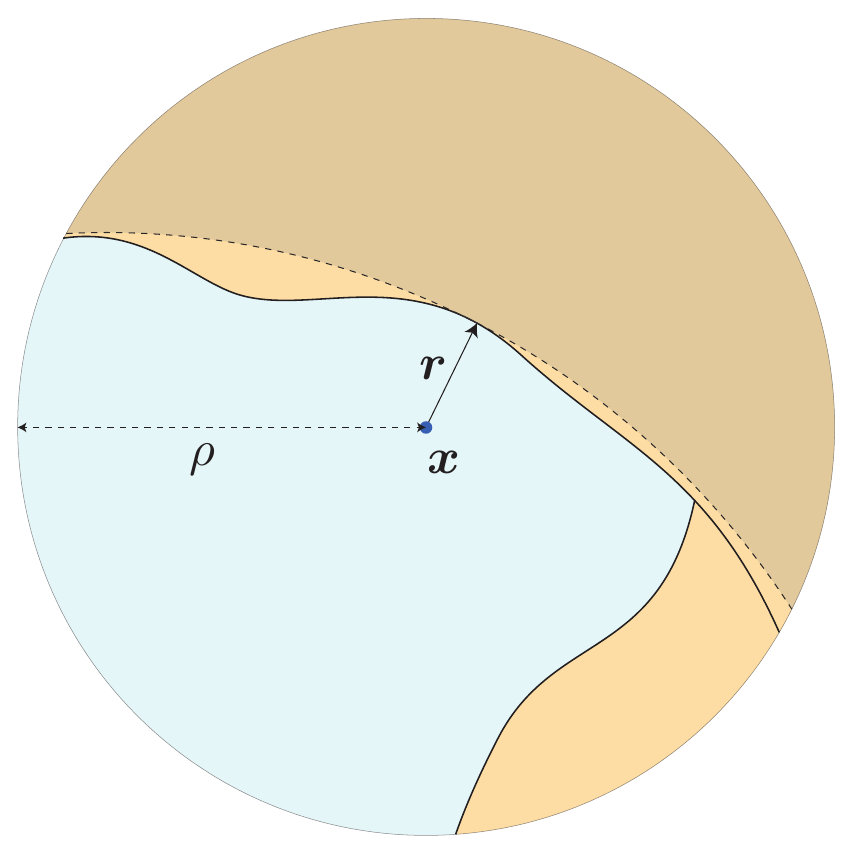}
  \caption{Curved decision boundary model $\mathscr{Q} (\x, \rho)$.}
  \label{fig:quadratic_model}
  \end{subfigure}
  \caption{\label{fig:linear_quadratic} Illustration of the decision boundary models considered in this paper. (a): For the flat decision boundary model, the set  $\{ \v: |\r(\x)^T \v| \leq \| \r(\x) \|_2^2 \}$ is illustrated (stripe). Note that for $\v$ taken outside the stripe (i.e., in the grayed area), we have $\hat{k}(\x+\v) \neq \hat{k} (\x)$ or $\hat{k}(\x-\v) \neq \hat{k} (\x)$ in the $\rho$ neighborhood. (b): For the curved decision boundary model, the any vector $\v$ chosen in the grayed area is classified differently from $\hat{k} (\x)$.}
\end{figure}

\vspace{-1mm}

\section{Robustness of classifiers with curved decision boundaries}

\label{sec:curved}

We now consider a model of the decision boundary in the vicinity of the data points that allows to leverage the \textit{curvature} of nonlinear classifiers.  Under this decision boundary model, we study the existence of universal perturbations satisfying $\hat{k} (\x+\v) \neq \hat{k} (\x)$ for most $\x \sim \mu$.\footnote{Unlike for linear classifiers, we now consider the problem of finding a universal \textit{vector} (as opposed to universal \textit{direction}) that fools most of the data points. This corresponds to the notion of universal perturbations first highlighted in \cite{moosavi2017universal}.}

\begin{figure}[ht]
\centering
  \begin{subfigure}{0.3\textwidth}
    \center
	\includegraphics[width=0.8\textwidth,page=2]{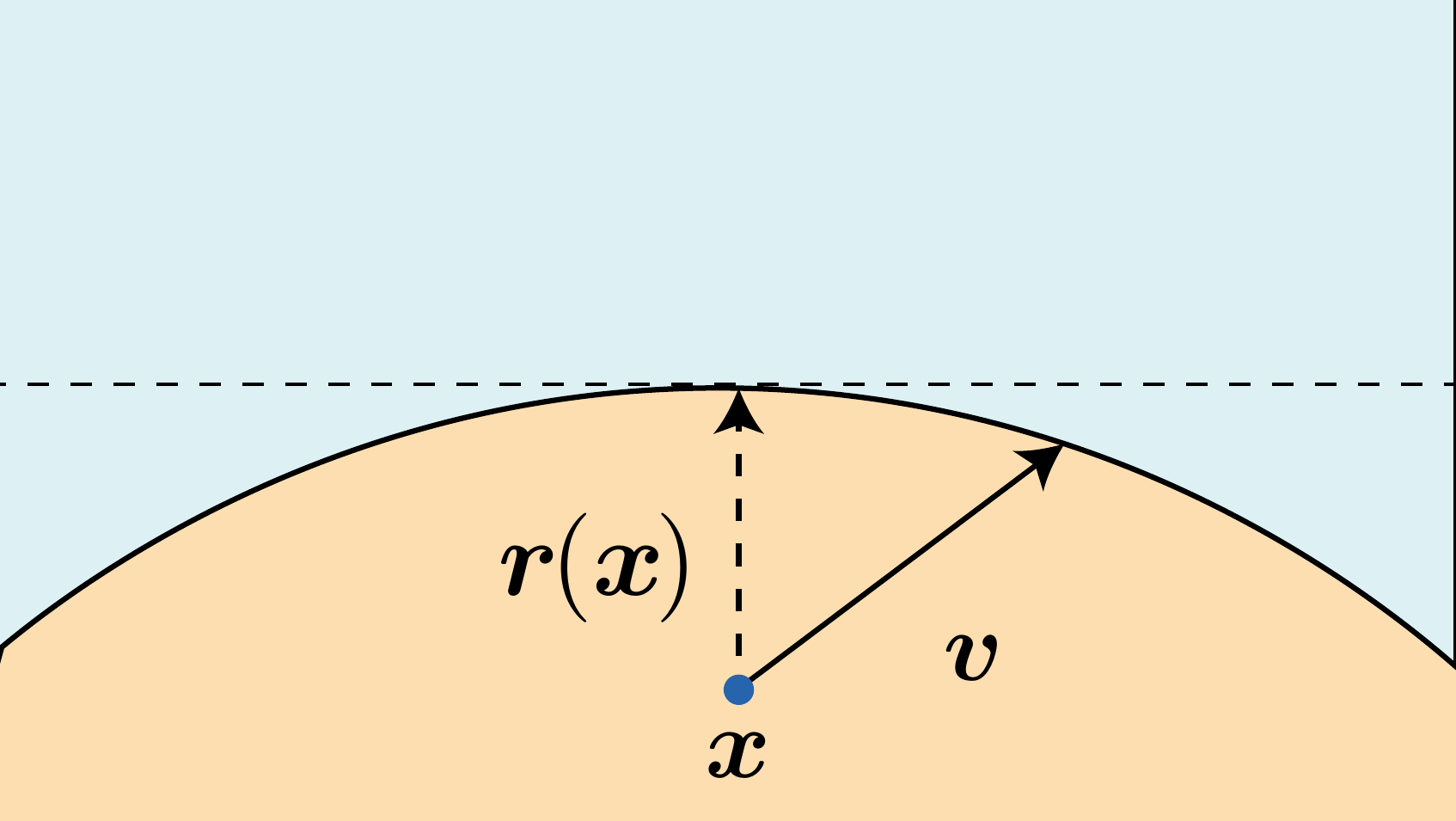}
 	\end{subfigure}
\begin{subfigure}{0.3\textwidth}
  \center
	\includegraphics[width=0.8\textwidth,page=4]{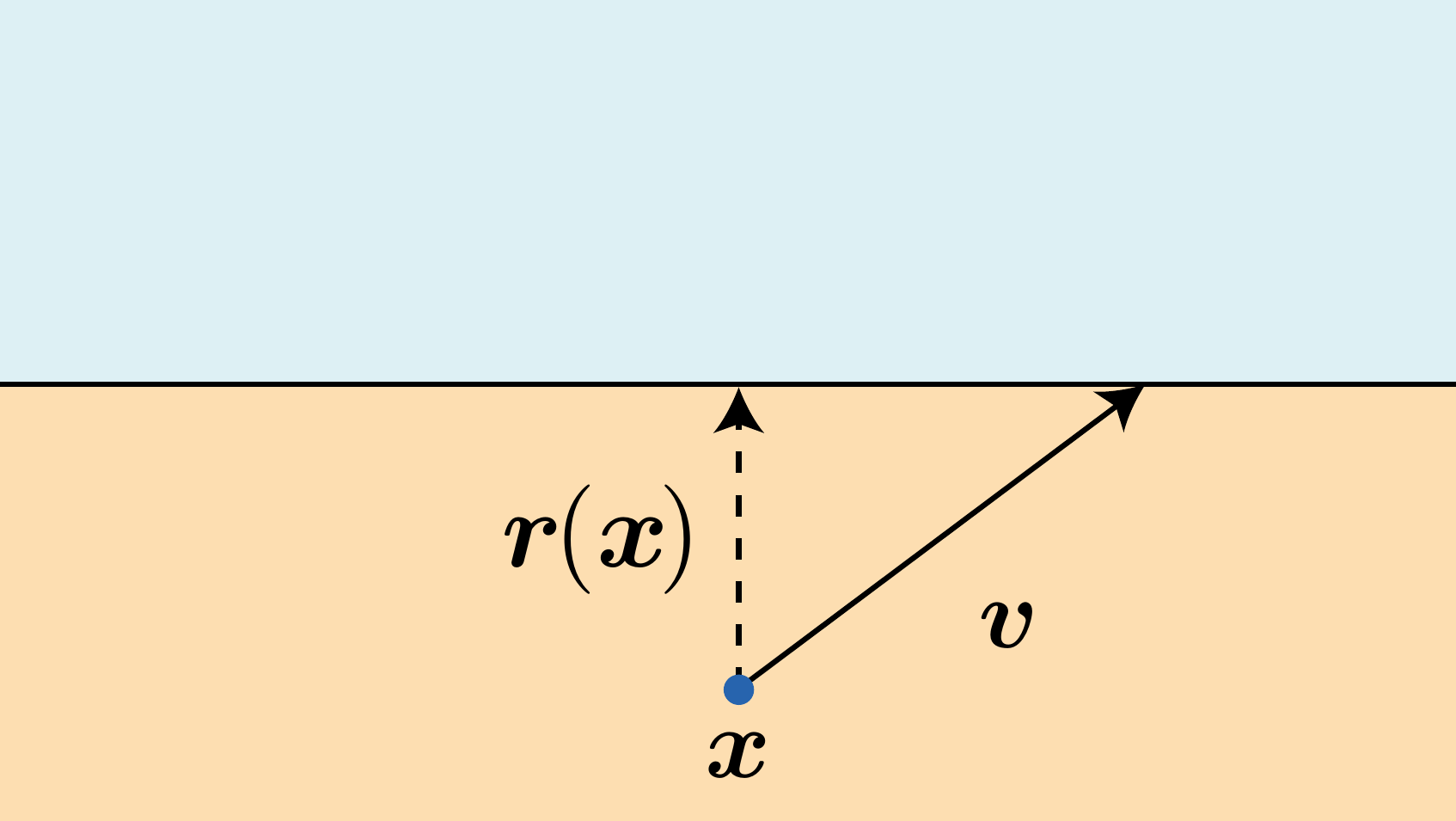}
 	\end{subfigure}
 	\begin{subfigure}{0.3\textwidth}
    \center
	\includegraphics[width=0.8\textwidth,page=6]{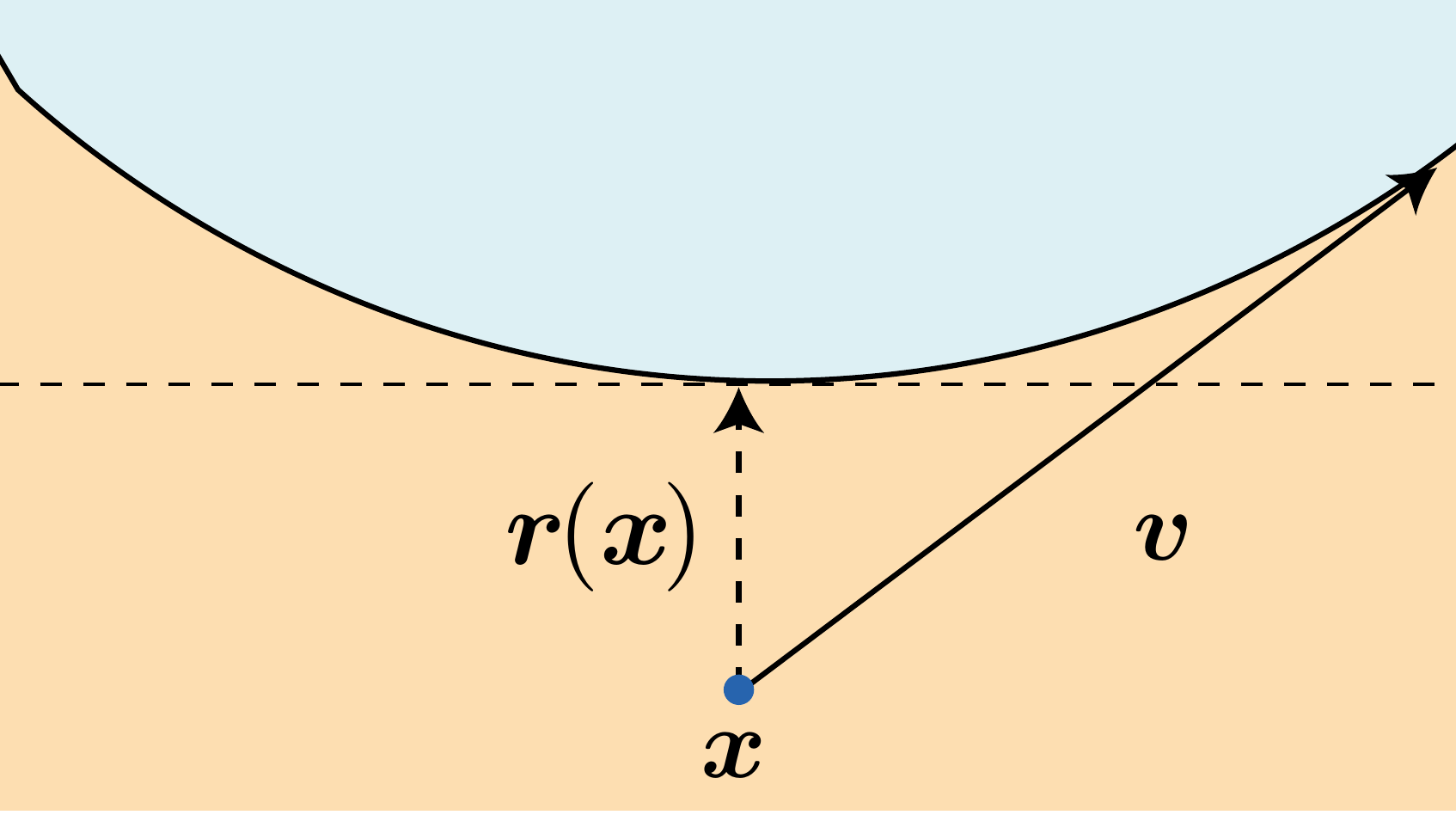}
 	\end{subfigure}
 	\caption{\label{fig:link_robustness_curvature} Link between robustness and curvature of the decision boundary. When the decision boundary is \textit{positively} curved (left), small universal perturbations are more likely to fool the classifier.}
\end{figure}

We start by establishing an informal link between \textit{curvature} of the decision boundary and \textit{robustness to universal perturbations}, that will be made clear later in this section. As illustrated in Fig. \ref{fig:link_robustness_curvature}, the norm of the required perturbation to change the label of the classifier along a specific direction $\v$ is smaller if the decision boundary is positively curved, than if the decision boundary is flat (or with negative curvature). It therefore appears from Fig. \ref{fig:link_robustness_curvature} that the existence of universal perturbations (when the decision boundary is curved) can be attributed to the existence of \textit{common} directions where the decision boundary is positively curved for many data points. In the remaining of this section, we formally prove the existence of universal perturbations, when there exists \textit{common} positively curved directions of the decision boundary. 

Recalling the definitions of Sec. \ref{sec:definitions}, a quadratic approximation of the decision boundary at $\z = \x+\r(\x)$ gives
$\x + \{ \v: (\v-\r(\x))^T H_{\z} (\v-\r(\x)) + \alpha_{x} \r(\x)^T (\v-\r(\x)) = 0 \}$,
where $H_{\z}$ denotes the Hessian of $F$ at $\z$, and $\alpha_x = \frac{\| \nabla F(\z) \|_2}{\| \r(\x) \|_2}$, with $F = f_i - f_j$. In this model, the second order information (encoded in the Hessian matrix $H_{\z}$) captures the curvature of the decision boundary. We assume a \textit{local} decision boundary model in the vicinity of datapoints $\x \sim \mu$, where the local classification region of $\x$ is bounded by a quadratic form. Formally, we assume that there exists $\rho > 0$ where the following condition holds for almost all $\x \sim \mu$:
\begin{align*}
\mathscr{Q}(\x, \rho): \forall \v \in B(\rho), (\v-\r(\x))^T H_{\z} (\v-\r(\x)) + \alpha_x \r(\x)^T(\v-\r(\x)) \leq 0 \implies \hat{k} (\x+\v) \neq \hat{k} (\x).
\end{align*}

An illustration of this quadratic decision boundary model is shown in Fig. \ref{fig:quadratic_model}.
The following result shows the existence of universal perturbations, provided a subspace $\mathcal{S}$ exists where the decision boundary has \textit{positive} curvature along most directions of $\mathcal{S}$: 

\begin{theorem}
\label{thm:curvature}
Let $\kappa > 0, \delta > 0$ and $m \in \bb{N}$. Assume that the quadratic decision boundary model $\mathscr{Q} \left(\x, \rho \right)$ holds for almost all $\x \sim \mu$, with
$\rho = \sqrt{\frac{2 \log(2/\delta)}{m}} \kappa^{-1} + \kappa^{-1/2}$.
Let $\S$ be a $m$ dimensional subspace such that
$$\pr{\v \sim \mathbb{S}}{\forall \u \in \Rbb^2, \alpha_x^{-1} \u^T H_{\z}^{\r(\x), \v} \u \geq \kappa \| \u \|_2^2} \geq 1 - \beta \text{ for almost all } \x \sim \mu,$$
where $H_{\z}^{\r(\x), \v} = \Pi^T H_{\z} \Pi$ with $\Pi$ an orthonormal basis of $\text{span} (\r(\x), \v)$, and $\Sbb$ denotes the unit sphere in $\mathcal{S}$.
Then, there is a universal perturbation vector $\v$ such that $\| \v \|_2 \leq \rho$ and $\pr{\x \sim \mu}{\hat{k}(\x+\v) \neq \hat{k}(\x)} \geq 1 - \delta - \beta$.
\end{theorem}
\begin{figure}[h]
  \center
  \begin{subfigure}{0.39\textwidth}
  \includegraphics[width=1.0\textwidth]{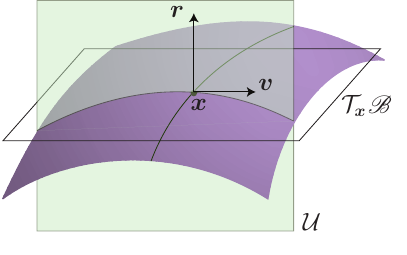}
\end{subfigure}
\hfill
  \begin{subfigure}{0.59\textwidth}
    \center
  \includegraphics[width=0.6\textwidth]{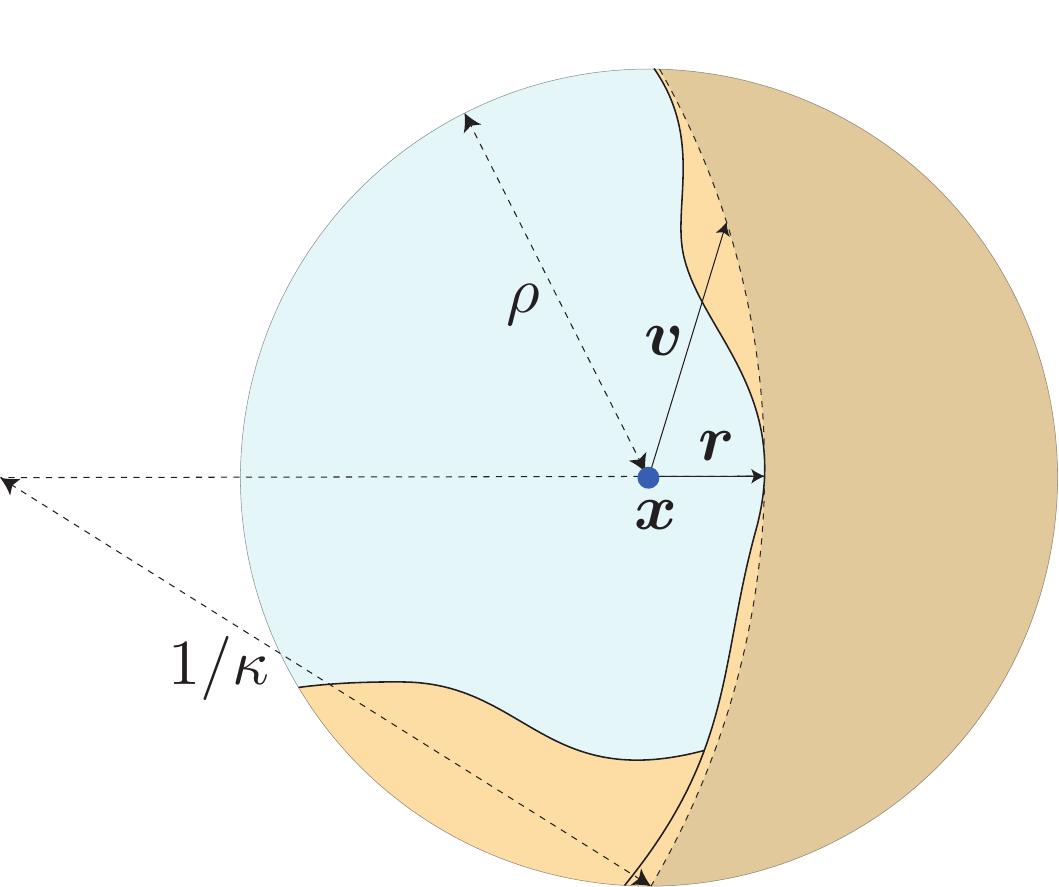}
\end{subfigure}
  \caption{\label{fig:normal_geometric}\textbf{Left:} Normal section $\mathcal{U}$ of the decision boundary, along the plane spanned by the normal vector $\r(\x)$ and $\v$. \textbf{Right:} Geometric interpretation of the assumption in Theorem \ref{thm:curvature}. Theorem \ref{thm:curvature} assumes that the decision boundary along normal sections $(\r(\x), \v)$ is locally (in a $\rho$ neighborhood) located \textit{inside} a disk of radius $1/\kappa$. Note the difference with respect to traditional notions of curvature, which express the curvature in terms of the osculating circle at $\x+\r(\x)$. The assumption we use here is more ``global''.}
  \label{fig:normal_section}
\end{figure}

The above theorem quantifies the robustness of classifiers to universal perturbations in terms of the curvature $\kappa$ of the decision boundary, along normal sections spanned by $\r(\x)$, and vectors $\v \in \mathcal{S}$ (see Fig. \ref{fig:normal_geometric} (\textit{left}) for an illustration of a normal section). 
Fig. \ref{fig:normal_geometric} (\textit{right}) provides a geometric illustration of the assumption of Theorem \ref{thm:curvature}. Provided a subspace $\mathcal{S}$ exists where the curvature of the decision boundary in the vicinity of datapoints $\x$ is \textit{positive}, Theorem \ref{thm:curvature} shows that universal perturbations can be found with a norm of approximately $\frac{\kappa^{-1}}{\sqrt{m}} + \kappa^{-1/2}$. Hence, when the curvature $\kappa$ is sufficiently large, the existence of small universal perturbations is guaranteed with Theorem \ref{thm:curvature}.\footnote{Theorem \ref{thm:curvature} should not be seen as a generalization of Theorem \ref{thm:theorem_S}, as the models are distinct. In fact, while the latter shows the existence of universal \textit{directions}, the former bounds the existence of universal \textit{perturbations}.}

\textbf{Remark 1.} We stress that Theorem \ref{thm:curvature} does \textit{not} assume that the decision boundary is curved in the direction of all vectors in $\mathbb{R}^d$, but we rather assume the existence of a subspace $\mathcal{S}$ where the decision boundary is positively curved (in the vicinity of natural images $\x$) across most directions in $\mathcal{S}$. 
Moreover, it should be noted that, unlike Theorem \ref{thm:theorem_S}, where the normals to the decision boundary are assumed to belong to a low dimensional subspace, no assumption is imposed on the normal vectors. Instead, we assume the existence of a subspace $\mathcal{S}$ leading to positive curvature, for points on the decision boundary in the vicinity of natural images.

\textbf{Remark 2.} Theorem \ref{thm:curvature} does not only predict the vulnerability of classifiers, but it also provides a constructive way to find such universal perturbations. In fact, \textit{random vectors} sampled from the subspace $\mathcal{S}$ are predicted to be universal perturbations (see the appendix for more details). In Section \ref{sec:experiments}, we will show that this new construction works remarkably well for deep networks, as predicted by our analysis.

\vspace{-3mm}

\section{Experimental results: universal perturbations for deep nets}
\label{sec:experiments}
We first evaluate the validity of the assumption of Theorem \ref{thm:curvature} for deep neural networks, that is the existence of a low dimensional subspace where the decision boundary is positively curved along most directions sampled from the subspace.
To construct the subspace, we find the directions that lead to large positive curvature in the vicinity of a given set of training points $\{ \x_1, \dots, \x_n \}$. 
We recall that principal directions $\v_1, \dots, \v_{d-1}$ at a point $\z$ on the decision boundary correspond to the \textit{eigenvectors} (with nonzero eigenvalue) of the matrix $H^t_{\z}$, given by $H^t_{\z} = P H_{\z} P$, where $P$ denotes the projection operator on the tangent to the decision boundary at $\z$, and $H_{\z}$ denotes the Hessian of the decision boundary function evaluated at $\z$ \cite{lee2009manifolds}. Common directions with large average curvature at $\z_i = \x_i + \r(\x_i)$ 
(where $\r(\x_i)$ is the minimal perturbation defined in Eq. (\ref{eq:adv_pert})) hence correspond to the eigenvectors of the average Hessian matrix $\overline{H} = n^{-1} \sum_{i=1}^n H^t_{\z_i}$. We therefore set our subspace, $\mathcal{S}_c$, to be the span of the first $m$ eigenvectors of $\overline{H}$, and 
 show that the subspace constructed in this way satisfies the assumption of Theorem \ref{thm:curvature}, for a deep net (LeNet) trained on CIFAR-10 with $n = 100$. To determine whether the decision boundary is positively curved in most directions of $\mathcal{S}_c$ (for unseen datapoints from the validation set), we compute the average curvature across random directions in $\mathcal{S}_c$ for points on the decision boundary, i.e. $\z = \x + \r(\x)$; the average curvature is formally given by
\[
\overline{\kappa}_{\mathcal{S}} (\x) = \ex{\v \sim \mathbb{S}}{\frac{(P \v)^T H_{\x} (P \v)}{\| P \v \|_2^2}},
\]
where $\mathbb{S}$ denotes the unit sphere in $\mathcal{S}_c$. 
In Fig. \ref{fig:curvature_dimension} (a), the average of $\overline{\kappa}_{\mathcal{S}} (\x)$ across points sampled from the \textit{validation set} is shown (as well as the standard deviation) in function of the subspace dimension $m$. Observe that when the dimension of the subspace is sufficiently small, the average curvature is strongly oriented towards positive curvature, which empirically shows the existence of this subspace $\mathcal{S}_c$ where the decision boundary is positively curved for most data points in the validation set. This empirical evidence hence suggests that the assumption of Theorem \ref{thm:curvature} is satisfied, and that universal perturbations hence represent \textit{random vectors} sampled from this subspace $\mathcal{S}_c$. 

To show this strong link between the vulnerability of universal perturbations and the \textit{positive curvature} of the decision boundary, we now visualize normal sections of the decision boundary of deep networks trained on ImageNet (CaffeNet) and CIFAR-10 (LeNet) in the direction of their respective universal perturbations. Specifically, we visualize normal sections of the decision boundary in the plane $(\r(\x), \v)$, where $\v$ is a universal perturbation computed using the universal perturbations algorithm of \cite{moosavi2017universal}. The visualizations are shown in Fig. \ref{fig:all_visualizations}. Interestingly, the universal perturbations belong to highly positively curved directions of the decision boundary, despite the absence of any geometric constraint in the algorithm to compute universal perturbations. To fool most data points, universal perturbations hence naturally seek \textit{common directions} of the embedding space, where the decision boundary is positively curved. These directions lead to very small universal perturbations, as highlighted by our analysis in Theorem \ref{thm:curvature}. It should be noted that such \textit{highly curved} directions of the decision boundary are rare, as random normal sections are comparatively very flat (see Fig. \ref{fig:all_visualizations}, second row). This is due to the fact that most principal curvatures are approximately zero, for points sampled on the decision boundary in the vicinity of data points. 

\begin{figure}[t!]
\centering
\includegraphics[width=0.8\textwidth]{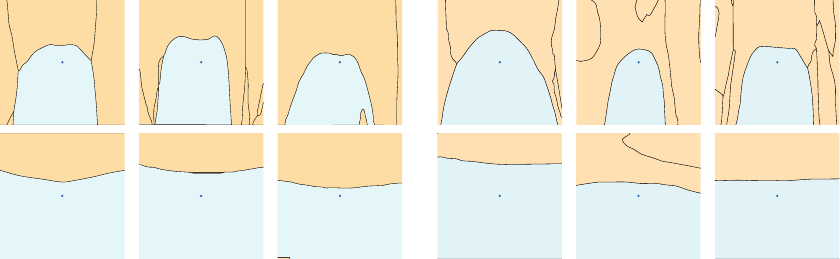}
\caption{\label{fig:all_visualizations} Visualization of normal cross-sections of the decision boundary, for ImageNet (CaffeNet, left) and CIFAR-10 (LeNet, right). 
\textbf{Top:} Normal cross-sections along $(\r(\x), \v)$, where $\v$ is the universal perturbation computed using the algorithm in \cite{moosavi2017universal}. \textbf{Bottom:} Normal cross-sections along $(\r(\x), \v)$, where $\v$ is a \textit{random} vector uniformly sampled from the unit sphere in $\mathbb{R}^d$.}
\label{fig:cross_sections}
\end{figure}

\begin{figure}[t!]
  \centering
  \begin{subfigure}[t]{0.40\textwidth}
    \includegraphics[width=\textwidth]{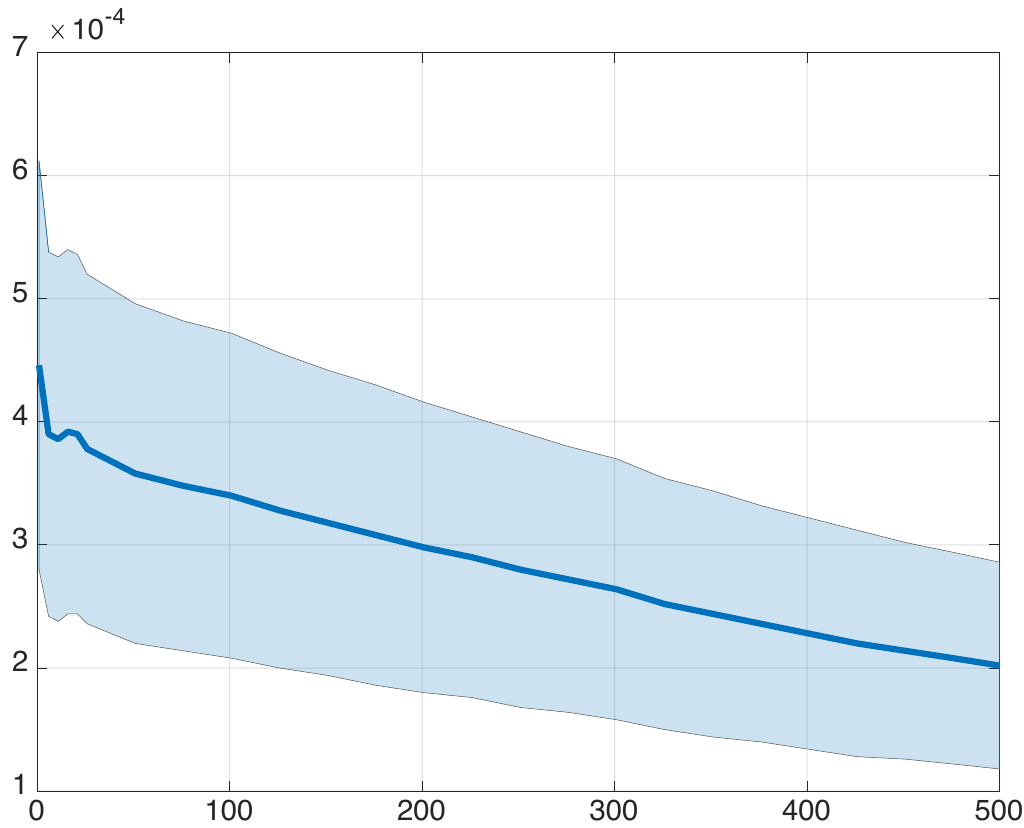}
    \caption{}
  \end{subfigure}
\begin{subfigure}[t]{0.40\textwidth}
  \includegraphics[width=\textwidth]{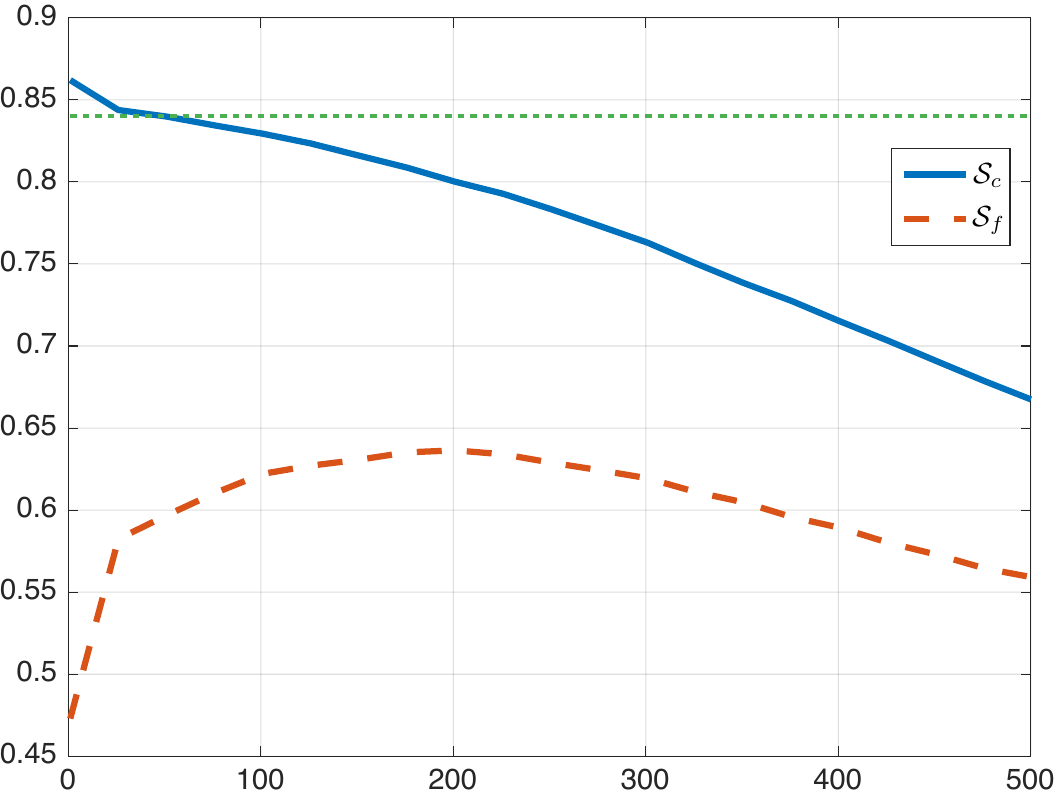}
  \caption{}
\end{subfigure}
\caption{\textbf{(a)} \label{fig:curvature_dimension} Average curvature $\overline{\kappa}_{\mathcal{S}}$, averaged over 1000 \textit{validation} datapoints, as a function of the subspace dimension. \textbf{(b)} \label{fig:compare_subspaces} Fooling rate of universal perturbations (on an unseen \textit{validation} set) computed using random perturbations in 1) $\mathcal{S}_c$: the subspace of positively curved directions, and 2) $\mathcal{S}_f$: the subspace collecting normal vectors $\r(\x)$. The dotted line corresponds to the fooling rate using the algorithm in \cite{moosavi2017universal}. $\mathcal{S}_f$ corresponds to the largest singular vectors corresponding to the matrix gathering the \textit{normal vectors} $\r(\x)$ in the training set (similar to the approach in \cite{moosavi2017universal}).}
\end{figure}
Recall that Theorem \ref{thm:curvature} suggests a novel way to procedure to generate universal perturbations; in fact, random perturbations from $\mathcal{S}_c$ are predicted to be universal perturbations.
To assess the validity of this result, Fig. \ref{fig:compare_subspaces} (b) illustrates the fooling rate of the universal perturbations (for the LeNet network on CIFAR-10) sampled uniformly at random from the unit sphere in subspace $\mathcal{S}_c$, and scaled to have a fixed norm ($1/5$th of the norm of the random noise required to fool most data points). We assess the quality of such perturbation by further indicating in Fig. \ref{fig:compare_subspaces} (b) the fooling rate of the universal perturbation computed using the algorithm in \cite{moosavi2017universal}.  
Observe that random perturbations sampled from $\mathcal{S}_c$ (with $m$ small) provide very poweful universal perturbations, fooling nearly $85\%$ of data points from the validation set. This rate is comparable to that of the algorithm in \cite{moosavi2017universal}, while using much less training points (only $n = 100$, while at least $2,000$ training points are required by \cite{moosavi2017universal}). The very large fooling rates achieved with such a simple procedure (random generation in $\mathcal{S}_c$) confirms that the curvature is the governing factor that controls the robustness of classifiers to universal perturbations, as analyzed in Section \ref{sec:curved}.
In particular, such high fooling rates cannot be achieved by only using the model of Section \ref{sec:locally_flat_model} (neglecting the curvature information), as illustrated in Fig. \ref{fig:compare_subspaces} (b).

\begin{figure}
  \centering
  \includegraphics[width=0.7\textwidth]{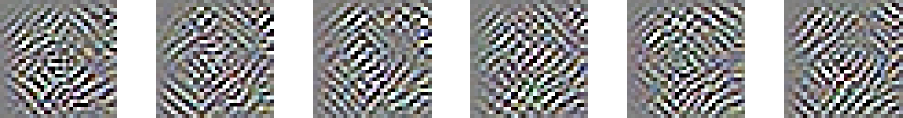}
  \caption{\label{fig:diversity}Diversity of universal perturbations randomly sampled from the subspace $\mathcal{S}_c$. The normalized inner product between two perturbations is less than $0.1$.}
\end{figure}
\begin{figure}[h!]
\centering
\includegraphics[width=0.4\textwidth]{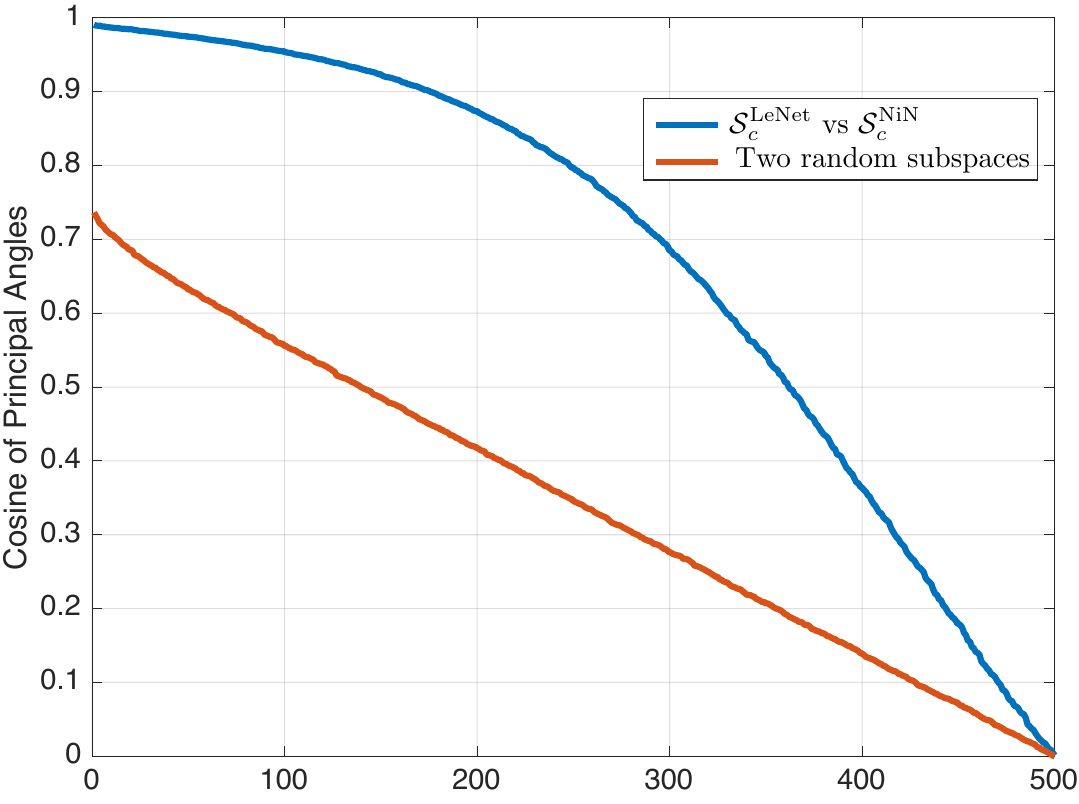}
\caption{\label{fig:principal_angles} Cosine of principal angles between $\mathcal{S}_c^{\text{LeNet}}$ and $\mathcal{S}_c^{\text{NiN}}$. For comparison, cosine of angles between two random subspaces is also shown.}
\end{figure}
The existence of this subspace $\mathcal{S}_c$ (and that universal perturbations are random vectors in $\mathcal{S}_c$) further explains the high \textit{diversity} of universal perturbations. Fig. \ref{fig:diversity} illustrates different universal perturbations for CIFAR-10 computed by sampling random directions from $\mathcal{S}_c$. The diversity of such perturbations justifies why re-training with perturbed images (as in \cite{moosavi2017universal}) does \textit{not} significantly improve the robustness of such networks, as other directions in $\mathcal{S}_c$ can still lead to universal perturbations, even if the network becomes robust to some directions. Finally, it is interesting to note that this subspace $\mathcal{S}_c$ is likely to be shared not only across datapoitns, but also different networks (to some extent).
To support this claim, Fig. \ref{fig:principal_angles} shows the cosine of the principal angles between subspaces $\mathcal{S}_c^{\text{LeNet}}$ and $\mathcal{S}_c^{\text{NiN}}$, computed for  LeNet and NiN \cite{lin2013} models. Note that the first principal angles between the two subspaces are very small, leading to shared directions between the two subspaces. A similar observation is made for networks trained on ImageNet in the appendix. The sharing of $\mathcal{S}_c$ across different networks explains the transferability of universal perturbations observed in \cite{moosavi2017universal}.

\vspace{-0.3cm}
\section{Conclusions}
\vspace{-0.1cm}
In this paper, we analyzed the robustness of classifiers to universal perturbations, under two decision boundary models. We showed that classifiers satisfying the flat decision model are not robust to universal directions, provided the normal vectors in the vicinity of natural images are correlated. While this model explains the vulnerability for e.g., linear classifiers, this model discards the curvature information, which is essential to fully analyze the robustness of deep nets to universal perturbations. We show in fact that classifiers with \textit{curved} decision boundaries are not robust to universal perturbations, provided the existence of a shared subspace along which the decision boundary is positively curved (for most directions). We empirically verify this assumption for deep nets. Our analysis hence explains the existence of universal perturbations, and further provides a purely geometric approach for computing such perturbations, in addition to explaining many properties of perturbations, such as their diversity.

Our analysis hence shows that to construct robust classifiers, it is key to \textit{suppress} this subspace of shared positive directions, which can possibly be done through regularization of the objective function. This will be the subject of future works.

\subsubsection*{Acknowledgments}
We gratefully acknowledge the support of NVIDIA Corporation with the donation of the Titan X Pascal GPU used for this research. This work has been partly supported by the Hasler Foundation, Switzerland, in the framework of the ROBERT project. A.F was supported by the Swiss National
Science Foundation under grant P2ELP2-168511. S.S. was supported by ONR N00014-17-1-2072 and ARO W911NF-15-1-0564.

\small{
\bibliographystyle{ieeetr}
\bibliography{bibliography}
}
\newpage
\appendix
\section{Proof of Theorem 1}

We first start by recalling a result from \cite{nips2016_ours}, which is based on \cite{dasgupta2003elementary}.

\begin{lemma}
\label{thm:jl_ours}
Let $\v$ be a random vector uniformly drawn from the unit sphere  $\mathbb{S}^{d-1}$, and $\P_m$ be the projection matrix onto the first $m$ coordinates. Then,
\begin{align}
\Pbb\left( \beta_1(\delta, m) \frac{m}{d} \leq \| \P_m \v \|_2^2 \leq \beta_2(\delta, m) \frac{m}{d} \right) \geq 1 - 2\delta,
\end{align}
with $\beta_1(\delta, m) = \max((1/e) \delta^{2/m}, 1-\sqrt{2(1-\delta^{2/m})}$, and $\beta_2(\delta, m) = 1 + 2 \sqrt{\frac{ \ln(1/\delta)}{m}} + \frac{2 \ln(1/\delta)}{m}$.
\end{lemma}

We use the above lemma to prove our result, which we recall as follows:

\begin{theorem}
Let $\xi \geq 0, \delta \geq 0$. Let $\mathcal{S}$ be an $m$ dimensional subspace such that$\| P_{\mathcal{S}} \r(\x) \|_2 \geq 1 - \xi \text{ for almost all } \x \sim \mu,$, where $P_{\mathcal{S}}$ is the projection operator on the subspace. Assume moreover that $\mathscr{L}_s \left(\x, \rho\right)$ holds for almost all $\x \sim \mu$, with $\rho = \frac{\sqrt{e m}}{\delta(1-\xi)}$. Then, there exists a universal noise vector $\v$, such that $\| \v \|_2 \leq \rho$ and $\pr{x\sim\mu}{\fe} \geq 1 - \delta.$
\end{theorem}

\begin{proof}
Define $\Sbb$ to be the unit sphere centered at $0$ in the subspace $\mathcal{S}$. Let $\rho = \frac{\sqrt{e m}}{\delta (1-\xi)}$, and denote by $\rho \Sbb$ the sphere scaled by $\rho$.
We have
\begin{align*}
& \ex{\v \sim \rho \Sbb}{ \pr{\x \sim \mu}{\fe} } \\
 = & \ex{\x \sim \mu}{\pr{ \v \sim \rho \Sbb}{\fe}} \\
\geq & \ex{\x \sim \mu}{\pr{\v \sim \rho \Sbb}{ |\r(\x)^T \v| - \| \r(\x) \|_2^2 \geq 0}} \\
 = & \ex{\x \sim \mu}{\pr{\v \sim \rho \Sbb}{ |(P_{\mathcal{S}} \r(\x) + P_{\mathcal{S}^\text{orth}} \r(\x))^T \v| - \| \r(\x) \|_2^2 \geq 0}},
\end{align*}
 where $P_{\mathcal{S}^\text{orth}}$ denotes the projection operator on the orthogonal of $\mathcal{S}$. Observe that $
 (P_{\mathcal{S}^\text{orth}} \r(\x))^T \v = 0$. Note moreover that $\| \r(\x) \|_2^2 = 1$ by assumption. Hence, the above expression simplifies to
 \begin{align*}
& \ex{\x \sim \mu}{\pr{\v \sim \rho \Sbb}{ |(P_{\mathcal{S}} \r(\x))^T \v| - 1 \geq 0}} \\
  = & \ex{\x \sim \mu}{\pr{\v \sim \Sbb}{ |(P_{\mathcal{S}} \r(\x))^T \v |  \geq \rho^{-1}}} \\
%
 \geq & \ex{\x \sim \mu}{\pr{\v \sim \Sbb}{ \left| \frac{(P_{\mathcal{S}} \r(\x))^T}{\| P_{\mathcal{S}} \r(\x) \|_2} \v \right| \geq \frac{\delta}{\sqrt{em}} }},
\end{align*}
where we have used the assumption of the projection of $\r(\x)$ on the subspace $\S$. Hence, it follows from Lemma \ref{thm:jl_ours} that
\[
 \ex{\v \sim \rho \Sbb}{ \pr{\x \sim \mu}{\fe} } \geq 1 - \delta.
\]
Hence, there exists a universal vector $\v$ of $\ell_2$ norm $\rho$  such that $\pr{\x \sim \mu}{\fe} \geq 1-\delta$.
\end{proof}

\section{Proof of Theorem 2}

\begin{theorem}
\label{thm:curvature}
Let $\kappa > 0, \delta > 0$ and $m \in \bb{N}$. Assume that the quadratic decision boundary model $\mathscr{Q} \left(\x, \rho \right)$ holds for almost all $\x \sim \mu$, with
$\rho = \sqrt{\frac{2 \log(2/\delta)}{m}} \kappa^{-1} + \kappa^{-1/2}$.
Let $\S$ be a $m$ dimensional subspace such that
$$\pr{\v \sim \mathbb{S}}{\forall \u \in \Rbb^2, \alpha_x^{-1} \u^T H_{\z}^{\r(\x), \v} \u \geq \kappa \| \u \|_2^2} \geq 1 - \beta \text{ for almost all } \x \sim \mu,$$
where $H_{\z}^{\r(\x), \v} = \Pi^T H_{\z} \Pi$ with $\Pi$ an orthonormal basis of $\text{span} (\r(\x), \v)$, and $\Sbb$ denotes the unit sphere in $\mathcal{S}$.
Then, there is a universal perturbation vector $\v$ such that $\| \v \|_2 \leq \rho$ and $\pr{\x \sim \mu}{\hat{k}(\x+\v) \neq \hat{k}(\x)} \geq 1 - \delta - \beta$.
\end{theorem}

\begin{proof}
Let $\x \sim \mu$.  We have
\begin{align*}
& \ex{\v \sim \rho \Sbb}{\pr{\x \sim \mu}{\hat{k}(\x+\v) \neq \hat{k}(\x)}} \\
= & \ex{\x \sim \mu}{\pr{\v \sim \rho \Sbb}{\hat{k}(\x+\v) \neq \hat{k}(\x)}} \\
\geq & \ex{\x \sim \mu}{\pr{\v \sim \rho \Sbb}{\alpha_{x}^{-1} (\v-\r)^T H_z (\v-\r) + \r^T(\v-\r) \geq 0}} \\
= & \ex{\x \sim \mu}{\pr{\v \sim \Sbb}{\alpha_{x}^{-1} (\rho \v-\r)^T H_z (\rho \v-\r) + \r^T(\rho \v-\r) \geq 0}}
\end{align*}
Using the assumptions of the theorem, we have
\begin{align*}
& \pr{\v \sim \Sbb}{\alpha_{x}^{-1} (\rho \v-\r)^T H_z (\rho \v-\r) + \r^T(\rho \v-\r) \leq 0} \\
\leq & \pr{\v \sim \Sbb}{\kappa \| \rho \v - \r \|_2^2 + \r^T (\rho \v - \r) \leq 0} + \beta \\
\leq & \pr{\v \sim \Sbb}{\rho (1-2\kappa) \v^T \r + \kappa \rho^2 + (\kappa - 1) \leq 0} + \beta \\
\leq & \pr{\v \sim \Sbb}{\rho (1-2\kappa) \v^T \r \leq - \epsilon} + \pr{\v \sim \Sbb}{\kappa \rho^2 + (\kappa - 1) \leq \epsilon} + \beta,
\end{align*}
for $\epsilon > 0$. The goal is therefore to find $\rho$ such that $\kappa \rho^2 + (\kappa - 1) \geq \epsilon$, together with $\pr{\v \sim \Sbb}{\rho (1-2\kappa) \v^T \r \leq - \epsilon} \leq \delta$. Let $\rho^2 = \frac{\epsilon + 1}{\kappa}$. Using the concentration of measure on the sphere \cite{matouvsek2002lectures}, we have
\[
\pr{\v \sim \Sbb}{\v^T \r \leq \frac{-\epsilon}{\rho(1-2\kappa)}} \leq 2 \exp\left( - \frac{m \epsilon^2}{2 \rho^2 (1-2\kappa)^2} \right).
\]
To bound the above probability by $\delta$, we set
$\epsilon = C \frac{\rho}{\sqrt{m}}$, where $C = \sqrt{2 \log(2/\delta)}$. We therefore choose $\rho$ such that
\[
\rho^2 = \kappa^{-1} \left( C \rho m^{-1/2} + 1 \right)
\]
The solution of this second order equation gives
\[
\rho = \frac{C \kappa^{-1} m^{-1/2} + \sqrt{\kappa^{-2} C^2 m^{-1} + 4 \kappa^{-1}}}{2}
		\leq C \kappa^{-1} m^{-1/2} + \kappa^{-1/2}.
\]
Hence, for this choice of $\rho$, we have by construction
\[
\pr{\v \sim \Sbb}{\alpha_{x}^{-1} (\rho \v-\r)^T H_z (\rho \v-\r) + \r^T(\rho \v-\r) \leq 0} \leq \delta + \beta.
\]
We therefore conclude that $\ex{\v \sim \rho \Sbb}{\pr{\x \sim \mu}{\hat{k}(\x+\v) \neq \hat{k}(\x)}} \geq 1 - \delta - \beta$. This shows the existence of a universal noise vector $\v \sim \rho \Sbb$ such that $\hat{k}(\x+\v) \neq \hat{k}(\x)$ with probability larger than $1-\delta-\beta$.
\end{proof}
\section{Transferability of universal perturbations}
\begin{figure}[h]
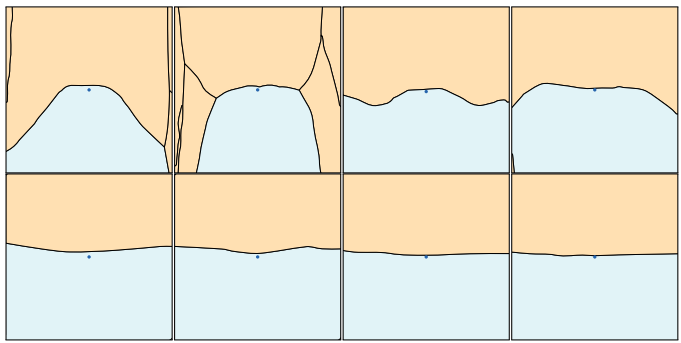

\centering
\newcounter{itr2}
\forloop{itr2}{5}{\value{itr2}<9}{
  \begin{subfigure}{0.11\textwidth}
    \includegraphics[width=1.1\textwidth,frame]{\arabic{itr2}.pdf}
  \end{subfigure}
}

\forloop{itr2}{5}{\value{itr2}<9}{
  \begin{subfigure}{0.11\textwidth}
    \includegraphics[width=1.1\textwidth,frame]{\arabic{itr2}_rand.pdf}
  \end{subfigure}
}
\caption{\label{fig:transferable}Transferability of the subspace $\mathcal{S}_c$ across different \textit{networks}. The first row shows normal cross sections along a fixed direction in $\mathcal{S}_c$ for VGG-16, with a subspace $\mathcal{S}_c$ computed with CaffeNet. Note the positive curvature in most cases. To provide a baseline for comparison, the second row illustrates normal sections along random directions.}
\label{fig:cross_sections}
\end{figure}

Fig. \ref{fig:transferable} shows examples of normal cross-sections of the decision boundary across a \textit{fixed} direction in $\mathcal{S}_c$, for the VGG-16 architecture (but where $\mathcal{S}_c$ is computed for \textit{CaffeNet}). Note that the decision boundary across this \textit{fixed} direction is positively curved for both networks, albeit computing this subspace for a distinct network. The sharing of $\mathcal{S}_c$ across different nets explains the transferability of universal perturbations observed in \cite{moosavi2017universal}.

\end{document}

%% file: settings_custom.tex
\usepackage{graphicx}
\usepackage{subcaption}
\usepackage{amsfonts,amssymb,amsmath,amsthm}
\usepackage{color}
\usepackage{url}
\usepackage{tablefootnote}
\usepackage{mathrsfs}
\usepackage{enumitem}
\usepackage{array}
\usepackage{algorithm,algpseudocode}
\usepackage{wrapfig}
\usepackage{xcolor}
\usepackage{mathtools}

\usepackage{nicefrac}
\usepackage{eufrak}
\usepackage{yfonts}

\newlength\myindent
\newcommand{\argmax}{\operatornamewithlimits{argmax}}

\newcommand{\pr}[2]{\underset{#1}{\Pbb}\left(#2\right)}
\newcommand{\ex}[2]{\underset{#1}{\Ebb}\left(#2\right)}

\newtheorem{lemma}{Lemma}
\newtheorem{theorem}{Theorem}

\newcommand{\Rbb}{\mathbb{R}}
\newcommand{\Sbb}{\mathbb{S}}

\newcommand{\bb}[1]{\mathbb{#1}}

\def\x{\boldsymbol{x}}

\def\z{\boldsymbol{z}}

\def\r{\boldsymbol{r}}
\def\w{\boldsymbol{w}}
\def\v{\boldsymbol{v}}

\newcommand{\lab}{\hat{k}}

\def\u{\boldsymbol{u}}
\def\S{\mathcal{S}}

\def\Sbb{\mathbb{S}}

\def\P{\mathbf{P}}
\def\Pbb{\mathbb{P}}
\def\Ebb{\mathbb{E}}

\makeatletter
\newtheorem*{rep@theorem}{\rep@title}
\newcommand{\newreptheorem}[2]{%
\newenvironment{rep#1}[1]{%
 \def\rep@title{#2 \ref{##1}}%
 \begin{rep@theorem}}%
 {\end{rep@theorem}}}
\makeatother

\newreptheorem{theorem}{Theorem}
\newreptheorem{lemma}{Lemma}
\newreptheorem{corollary}{Corollary}

\newcommand{\fe}{\hat{k} (\x+\v) \neq \hat{k} (\x) \text{ or } \hat{k}(\x-\v) \neq \hat{k} (\x)}

\DeclarePairedDelimiterX{\infdivx}[2]{}{}{%
  #1\;\delimsize\|\;#2%
}